\documentclass{article}

\usepackage[a4paper]{geometry}
\newgeometry{
    textheight=9in,
    textwidth=5.5in,
    top=1in,
    headheight=12pt,
    headsep=25pt,
    footskip=30pt
    }

\usepackage[utf8]{inputenc} %
\usepackage[T1]{fontenc}  
\usepackage{lmodern}  %
\usepackage{hyperref}       %
\usepackage{url}            %
\usepackage{booktabs}       %
\usepackage{amsfonts}       %
\usepackage{nicefrac}       %
\usepackage{microtype}      %
\usepackage{amsmath}
\usepackage{graphicx}
\usepackage{verbatim}
\usepackage{graphicx}
\usepackage{caption}
\usepackage{floatrow}
\usepackage{subfig}
\usepackage[utf8]{inputenc}
\usepackage[english]{babel}
\usepackage{tikz}
\usetikzlibrary{calc}

\hypersetup{
    colorlinks,
    linkcolor={red},
    citecolor={blue},
    urlcolor={blue}
}
 
\usepackage{amsthm}

\usepackage[detect-all]{siunitx}
\usepackage{array}
\usepackage{adjustbox, multirow}
\usepackage{breqn}

\newcommand{\tensor}[1]{\mathcal{#1}}
\newcommand{\jump}[1]{\ensuremath{[\![#1]\!]} }
\newcommand\sbullet[1][.5]{\mathbin{\vcenter{\hbox{\scalebox{#1}{$\bullet$}}}}}
\renewcommand{\ll}{\langle\!\langle}
\renewcommand{\gg}{\rangle\!\rangle}

\usepackage{todonotes}

\title{Tensor Regression Networks with various Low-Rank Tensor Approximations}

\author{
  Xingwei~Cao$^{\dagger}$, Guillaume~Rabusseau$^{\ddagger }$\\
 Tensor Learning Unit, RIKEN Center for AIP, Tokyo, Japan$^\dagger$\\
 School of Computer Science, McGill University, Quebec, Canada$^\ddagger $\\
  \texttt{ \{xingwei.cao, guillaume.rabusseau\}@mail.mcgill.ca}
}
\date{}

\newcommand{\eg}{e.g.\ }

\theoremstyle{plain}
\newtheorem{thm}{Theorem}[]
\newtheorem{lem}[thm]{Lemma}
\newtheorem{prop}[thm]{Proposition}

\theoremstyle{definition}
\newtheorem{defn}{Definition}[]

\begin{document}

\maketitle

\begin{abstract}

Tensor regression networks achieve high compression rate of neural networks while having slight impact on performances.
They do so by imposing low tensor rank structure on the weight matrices of fully connected layers.
In recent years, tensor regression networks have been investigated from the perspective of their compressive power, however, the regularization effect of enforcing low-rank tensor structure has not been investigated enough.
We study tensor regression networks using various low-rank tensor approximations, aiming to compare the compressive and regularization power of different low-rank constraints.
We evaluate the compressive and regularization performances of the proposed model with both deep and shallow convolutional neural networks.
The outcome of our experiment suggests the superiority of Global Average Pooling Layer over Tensor Regression Layer when applied to deep convolutional neural network with CIFAR-10 dataset.
On the contrary, shallow convolutional neural networks with tensor regression layer and dropout achieved lower test error than both Global Average Pooling and fully-connected layer with dropout function when trained with a small number of samples.

\end{abstract}

\section{Introduction} \label{sc:intro}
Tensor has been attracting increasing interests from the machine learning community over past decades. One of the reasons for such appreciation towards tensor is the natural representation of multi-modal data using the tensor structure. Such multi-modal dataset are often encountered in scientific fields including image analysis \cite{liu2013tensor}, signal processing \cite{cichocki2015tensor} and spatio-temporal analysis \cite{bahadori2014fast,yu2016learning}.
Tensor methods allow statistical models to efficiently learn multilinear relationship between inputs and outputs by leveraging multilinear algebra and efficient low-rank constraints. The low-rank constraints on higher-order multivariate regression can be interpreted as a regularization technique. As shown in \cite{rabusseau2016low}, efficient low-rank multilinear regression model with tensor response can improve the performance of regression.

Incorporating tensor methods into deep neural network has become a prominent area of studies. In particular, over the past decade, tensor decomposition and approximation algorithms have been introduced to deep neural networks, notably for 1) efficient compression of the model with low-rank constraints \cite{novikov2015tensorizing} and 2) leveraging the multi-modal structure of the high-dimensional dataset \cite{kossaifi2017tensor}.
For illustration, Kossaifi et al. proposed tensor regression layer~(TRL) which replaces the vectorization operation and fully-connected layers of Convolutional Neural Networks~(CNNs) with higher-order multivariate regression \cite{kossaifi2017tensor}. The advantage of such replacement is the high compression rate of the model while preserving multi-modal information of dataset by enforcing efficient low-rank constraints. Given such high-dimensional dataset, the vectorization operation will lead to the loss of multi-modal information. The higher-level dependencies among various modes are lost when the data is mapped to a linear space. For instance, applying flattening operation to a colored image ($3$rd-order tensor) will remove the relationship between the red-channel and the blue-channel. Tensor regression layer is able to capture such multi-modal information by performing multilinear regression tasks between the output of the last convolutional layer and the softmax.

Following~\cite{kossaifi2017tensor}, we investigate the property and performance of tensor regression layers from the perspectives of regularization and compression. We interpret low-rank constraints as a regularization technique for higher-order multivariate regression and enforce low-rank constraints on the weight tensor between output tensors of CNN and output vectors. Furthermore, we compare tensor regression layer with various tensor decomposition approximations. We aim to provide a comparative insight on different low-rank constraints that can be enforced on higher-order multivariate regression. We compare the performances of TRL using Tucker, CP and Tensor Train decompositions in a small standard CNN on MNIST and Fashion-MNIST. We also investigate such comparison in Residual Networks (ResNet)~\cite{he2016deep, he2016identity} on CIFAR-10. To investigate the regularization effect, we employed shallow CNNs and trained them with different numbers of training samples and compare the performances.

We show that a compression rate of 54$\times$ can be achieved using TT decomposition with a sacrifice of accuracy less than 0.3\% with respect to the weight matrix of a 32-Layer Residual Network with fully-connected layer on CIFAR-10 dataset. Surprisingly, we also show that an even better compression rate with a smaller loss in accuracy on CIFAR-10 can be achieved by simply using global average pooling~(GAP) followed by a small fully connected layer. However, using the same trick on the smaller CNN on MNIST led to very poor results.

The remaining of this paper is organized as follows. We start by reviewing background knowledge of multilinear algebra and tensor decomposition formats in Section~\ref{sc:back}. In Section~\ref{sc:trl}, we present and investigate tensor regression layer with different tensor decomposition formats. We show that global average pooling~GAP) layer is a special case of TRL with Tucker decomposition in Section~\ref{sc:gap}. In Section~\ref{sc:observation_on_rank_constraints} we present a simple analysis of low-rank constraints showing how particular choices of the tensor rank parameters can drastically affect the expressiveness of the network.  We demonstrate empirical performance of low-rank TRL in Section~\ref{sc:exp} followed by discussion and conclusion of our work in Section~\ref{sc:conc}.

\section{Background} \label{sc:back}

\begin{figure}
\centering
	\subfloat[Tensor Train Decomposition.]{
	\scalebox{0.75}[0.75]{\begin{tikzpicture}
	    \begin{scope}[every node/.style={minimum size=3.0em}]
	    \node[draw, shape=circle] (g1) at 			(0,0) {$\tensor{G}^{(1)}$};
	    \node[draw, shape=circle] (g2) at 			(2.0,0) {$\tensor{G}^{(2)}$};
	    \node[draw, shape=circle] (g3) at 			(4.0,0) {$\tensor{G}^{(3)}$};
	    \node[draw, shape=circle] (g4) at 			(6.0,0) {$\tensor{G}^{(4)}$};
    
	    \draw (g1) -- (g2) node[midway,above] {$R_1$};
	    \draw (g2) -- (g3) node[midway,above] {$R_2$};
	    \draw (g3) -- (g4) node[midway,above] {$R_3$};
	    \draw (g1) -- +(0,-1.5cm) node[midway, right = -5pt]{$I_1$};
	    \draw (g2) -- +(0,-1.5cm) node[midway, right = -5pt]{$I_2$};
	    \draw (g3) -- +(0,-1.5cm) node[midway, right = -5pt]{$I_3$};
	    \draw (g4) -- +(0,-1.5cm) node[midway, right = -5pt]{$I_4$};
	    
	    \end{scope}
	\end{tikzpicture}}
	} \hspace{30pt} \subfloat[Tucker Decomposition.]{
	\scalebox{0.75}[0.75]{\begin{tikzpicture}
	    \begin{scope}[every node/.style={minimum size=3.0em}]
	    \node[draw, shape=circle] (g) at 			(3.0,0) {$\tensor{G}$};
	    \node[draw, shape=circle] (u1) at 			(0,-1.5) {$\mathbf{U}^{(1)}$};
	    \node[draw, shape=circle] (u2) at 			(2.0,-1.5) {$\mathbf{U}^{(2)}$};
	    \node[draw, shape=circle] (u3) at 			(4.0,-1.5) {$\mathbf{U}^{(3)}$};
	    \node[draw, shape=circle] (u4) at 			(6.0,-1.5) {$\mathbf{U}^{(4)}$};
	    
        \draw [-, out=180, in=70, bend right=20, looseness = 1.5] (g.west) to node[above]{$R_1$}(u1.north);
	    \draw (g) edge[out=200,in=90] (u2) node[midway, right=45pt, below]{$R_2$};
	    \draw (g) edge[out=-20,in=90] (u3) node[midway, right=125pt, below]{$R_3$};
        \draw [-, out=0, in=110, bend left=20, looseness = 1.5] (g.east) to node[above]{$R_4$}(u4.north);
        
	    \draw (u1) -- +(0,-1.5cm) node[midway, right = -5pt]{$I_1$};
	    \draw (u2) -- +(0,-1.5cm) node[midway, right = -5pt]{$I_2$};
	    \draw (u3) -- +(0,-1.5cm) node[midway, right = -5pt]{$I_3$};
	    \draw (u4) -- +(0,-1.5cm) node[midway, right = -5pt]{$I_4$};
        
	    \end{scope}
	\end{tikzpicture}}
	}

\caption{Tensor network representations of Tucker and Tensor Train decomposition of an input tensor $\tensor{X}$ in a space $\mathbb{R}^{I_1 \times I_2 \times I_3 \times I_4}$. Circular nodes and edges represent tensors and contraction operation between two tensors respectively.} \label{fig:tt_t_basic}
\end{figure}
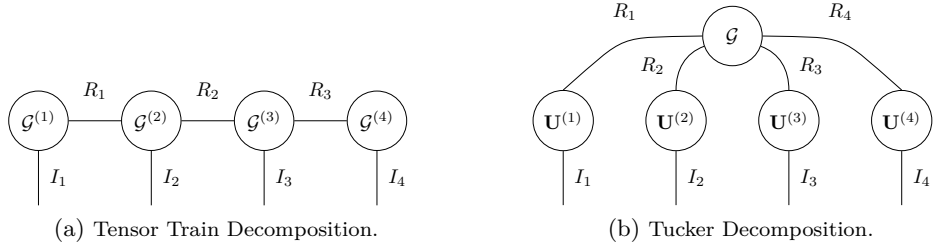

\subsection{Tensor Algebra}

We begin with a concise review of notations and basics of tensor algebra. For a more comprehensive review, we refer the reader to \cite{kolda2009tensor}. Throughout this paper, a vector is denoted by boldface lowercase letter, \eg $\mathbf{v}\in \mathbb{R}^{I_1}$. Matrices and higher-order tensors are denoted by boldface uppercase and calligraphic letters respectively, \eg  $\mathbf{M} \in \mathbb{R}^{I_1 \times I_2}$ and $\tensor{T}\in \mathbb{R}^{I_1 \times I_2 \times \cdots \times I_N}$. Given an $N$th-order tensor $\tensor{T}\in\mathbb{R}^{I_{1}\times I_{2}\times\cdots\times I_{N}}$, its $(i_1, i_2, \ldots, i_N)$th entry is denoted by $\tensor{T}_{i_{1},i_{2},\ldots , i_{N}}$ or $(\tensor{T})_{i_{1}, i_{2}, \cdots , i_{N}}$, where $i_{n}=1, 2,\ldots, I_{n}, \forall n\in[1,N]$. The notation $[N,M]$ denotes the range of integers from $N$ to $M$ inclusive.
Given a $3$rd order tensor $\tensor{T} \in\mathbb{R}^{I\times J\times K}$, its \emph{slices} are the matrices obtained by fixing all but two indices; the horizontal, lateral and frontal slices of  $\tensor{T}$ are denoted by $\mathbf{T}_{i,:,:}\in\mathbb{R}^{J\times K}$,  $\mathbf{T}_{:,j,:}\in\mathbb{R}^{I\times K}$ and  $\mathbf{T}_{:,:,k}\in\mathbb{R}^{I\times J}$ respectively. 
Similarly, the mode-\emph{n} fibers of  $\tensor{T}$ are the vectors obtained by fixing every index but the \emph{n}-th one. The mode-\emph{n} matricization or mode-\emph{n} unfolding of a tensor $\tensor{T}$ is the matrix having its mode-$n$ fibers as columns and is denoted by $\mathbf{T}_{(n)}$. 
Given $n$ vectors $\mathbf{v}_1\in \mathbb{R}^{I_1},\mathbf{v}_2\in \mathbb{R}^{I_2},\dots,\mathbf{v}_n \in \mathbb{R}^{I_N}$, the outer product of these vectors is denoted by $\mathbf{v}_1\circ \mathbf{v}_2\circ \cdots \circ \mathbf{v}_n\in\mathbb{R}^{I_{1}\times I_{2}\times\cdots\times I_{N}}$ and is defined by $\left( \mathbf{v}_1\circ \mathbf{v}_2\circ \cdots \circ \mathbf{v}_n \right)_{i_1, i_2, \cdots ,i_n} = (\mathbf{v}_1)_{i_1}(\mathbf{v}_2)_{i_2}\cdots (\mathbf{v}_n)_{i_n}$ for all $i_k \in [1,I_k]$ where $k \in [1,N]$. An~N-th order tensor $\tensor{X}\in\mathbb{R}^{I_1 \times \cdots \times I_N}$ is called rank-one if it can be written as the outer product of N vectors (i.e. $\tensor{X}=\mathbf{v}_1\circ \mathbf{v}_2\circ \cdots \circ \mathbf{v}_N$).
The $n$-mode product of a tensor $\tensor{X}\in\mathbb{R}^{I_{1}\times I_{2}\times\cdots\times I_{N}}$ with a matrix $\mathbf{U}\in\mathbb{R}^{J\times I_{n}}$ is denoted by $\tensor{X}\times_n \mathbf{U}\in \mathbb{R}^{I_1 \times \cdots \times I_{n-1} \times J \times I_{n+1} \times \cdots \times I_N}$ and is defined by 
$$(\tensor{X}\times_n\mathbf{U})_{i_1, \cdots , i_{n-1} , j , i_{n+1},\cdots ,i_N} = \sum_{i_n =1}^{I_n} \tensor{X}_{i_1, i_2, \cdots, i_N}\, u_{j , i_n}$$
for all $i_k\in [1,I_k]$, $j\in[1,J]$ where $k \in [1,N]$. Similarly, we denote an $n$-mode product of a tensor $\tensor{X}\in\mathbb{R}^{I_{1}\times I_{2}\times\cdots\times I_{N}}$ and a vector $\mathbf{v}\in\mathbb{R}^{I_n}$ by $\tensor{X}\sbullet[.75]_n \mathbf{v} \in \mathbb{R}^{I_1 \times \cdots \times I_{n-1}  \times I_{n+1} \times \cdots \times I_N}$ for all $n\in[1,N]$ and it is defined by $\tensor{X}\sbullet[.75]_n \mathbf{v} = \tensor{X}\times_n \mathbf{v}^T$.

The \emph{Kronecker product} of matrices $\mathbf{A}\in\mathbb{R}^{I\times J}$ and $\mathbf{B}\in\mathbb{R}^{K\times L}$ is the block matrix $(\mathbf{A}_{i,j}\mathbf{B})_{i,j}$ of size $IK\times JL$ and is denoted by $\mathbf{A}\otimes \mathbf{B}$. Given matrices $\mathbf{A}$ and $\mathbf{B}$, both of  size ${I\times J}$, their \emph{Hadamard product}~(or component-wise product) is denoted by $\mathbf{A}*\mathbf{B}$ and defined by $(\mathbf{A}*\mathbf{B})_{i,j} =\mathbf{A}_{i,j}\mathbf{B}_{i,j}$. The \emph{Khatri-Rao product} $\mathbf{A} \odot \mathbf{B}$ of matrices $\mathbf{A}\in\mathbb{R}^{I\times K}$ and $\mathbf{B}\in\mathbb{R}^{J\times K}$ is the ${IJ\times K}$ matrix  defined by
\begin{equation}
  \mathbf{A} \odot \mathbf{B} = \left[
    \begin{array}{cccc}
      \mathbf{a}_1 \otimes \mathbf{b}_1 & \mathbf{a}_2 \otimes \mathbf{b}_2 & \cdots & \mathbf{a}_K \otimes \mathbf{b}_K
    \end{array}
  \right]
\end{equation}
where $\mathbf{a}_i$~(resp. $\mathbf{b}_i$) denotes the $i$th column of  $\mathbf{A}$~(resp. $\textbf{B}$).

\subsection{Various Tensor Decompositions}
In this section we present three of the commonly used tensor decomposition formats: Candecomp/Parafac, Tucker and Tensor-Train.

{\bf CP decomposition.} \hspace{5pt} The CP decomposition \cite{carroll1970analysis, harshman1970foundations}  approximates a tensor with a summation of rank-one tensors \cite{kolda2009tensor}. The rank of the decomposition is simply the number of rank-one tensors used to approximate the input tensor: given an input tensor $\tensor{X}\in\mathbb{R}^{I_{1}\times I_{2}\times\cdots\times I_{N}}$, its approximation with a CP decomposition of rank $R$  is defined by

\begin{equation} \label{eq:cp_def}
\tensor{X} \approx \sum_{j=1}^{R}{\mathbf{a}^{(j)}_1 \circ \cdots \circ \mathbf{a}^{(j)}_N} = \jump{\mathbf{A}^{(1)}, \mathbf{A}^{(2)}, \cdots , \mathbf{A}^{(N)} }.
\end{equation}

In Eq.~\eqref{eq:cp_def}, $\jump{\mathbf{A}^{(1)}, \mathbf{A}^{(2)}, \cdots , \mathbf{A}^{(1)} }$ denotes the CP approximation of $\tensor{X}$ where each matrix $\mathbf{A}^{(i)}\in\mathbb{R}^{I_i\times R}$ consists of the $R$ column vectors $\mathbf{a}_i^{(j)}$ for $j\in[1,R]$.

We have the following useful expression of Eq.~\eqref{eq:cp_def} in terms of the matricization of $\tensor{X}$:

\begin{equation} \label{eq:cp_mtr}
\tensor{X}_{(n)} \approx \mathbf{A}^{(n)}\left( \mathbf{A}^{(N)} \odot \cdots \odot \mathbf{A}^{(n+1)} \odot \mathbf{A}^{(n-1)} \odot \cdots \odot \mathbf{A}^{(1)} \right)^T
\end{equation}

{\bf Tucker decomposition.} \hspace{5pt} The Tucker decomposition approximates a tensor $\tensor{X}\in\mathbb{R}^{I_{1}\times I_{2}\times\cdots\times I_{N}}$ by the product of a core tensor $\tensor{G}\in\mathbb{R}^{R_{1}\times R_{2}\times\cdots\times R_{N}}$ and $N$ factor matrices $\mathbf{U}^{(i)}\in\mathbb{R}^{I_i \times R_i}$ for $i\in[1,N]$:
\begin{equation} \label{eq:tc_def}
\tensor{X} \approx \tensor{G} \times_1 \mathbf{U}^{(1)} \times_2 \cdots \times_N \mathbf{U}^{(N)} = \jump{\tensor{G}; \mathbf{U}^{(1)}, \cdots , \mathbf{U}^{(N)}}
\end{equation}

The matricization of $\tensor{X}$ from Eq.~\eqref{eq:tc_def} can be written as

\begin{equation} \label{eq:tc_mtr}
\tensor{X}_{(n)} \approx \mathbf{U}^{(n)} \tensor{G}_{(n)} \left( \mathbf{U}^{(N)} \otimes \cdots \mathbf{U}^{(n+1)} \otimes \mathbf{U}^{(n-1)} \otimes \cdots \otimes \mathbf{U}^{(1)} \right)^T
\end{equation}

The tuple $(R_1 , \cdots , R_N )$ is the rank of the Tucker decomposition and determines the size of the core tensor $\tensor{G}$. An example of a Tucker approximation of a fourth order tensor is given in Figure \ref{fig:tt_t_basic}.

{\bf Tensor train decomposition.} \hspace{5pt} The tensor train (TT) decomposition \cite{oseledets2011tensor} provides a space-efficient representation for higher-order tensors. It approximates a tensor $\tensor{X}\in\mathbb{R}^{I_{1}\times I_{2}\times\cdots\times I_{N}}$ with the product of third order tensors $\tensor{G}^{(i)}\in\mathbb{R}^{R_{i-1} \times  I_i \times R_{i}}$ called core tensors or simply cores. The rank of the TT decomposition is the tuple $(R_0, R_1,
\cdots,R_N)$ where $R_0 = R_N = 1$. 

Given a tensor $\tensor{X}\in\mathbb{R}^{I_{1}\times I_{2}\times\cdots\times I_{N}}$, the approximation by TT decomposition is defined as

\begin{equation} \label{eq:tt_def}
\tensor{X}_{i_{1}, i_{2}, \ldots , i_{N}} \approx \mathbf{G}^{(1)}_{i_1 , :} \times \mathbf{G}^{(2)}_{:, i_2, :} \times \cdots \times \mathbf{G}^{(N)}_{:, i_N}
=\ \ll \tensor{G}^{(1)}, \cdots , \tensor{G}^{(N)} \gg
\end{equation}
where $\times$ denotes the matrix product.

In order to express Eq.~\eqref{eq:tt_def} in terms of matricizations of $\tensor{X}$, we first define the following contraction operation on core tensors.

\begin{defn} \label{df:tt_core}
Given a set of core tensors $\tensor{G}^{(i)}$ in Eq.~\eqref{eq:tt_def} for $i\in[1,N]$, we define $\tensor{G}^{<n}$ as the product of core tensors $\tensor{G}^{(j)}$ for $j\in[1,n-1]$:
\begin{equation} \label{eq:tt_seq_def}
\tensor{G}^{<n}_{i_1 , \cdots , i_{n-1} , i_{n}} = \mathbf{G}^{(1)}_{i_1 , :}  \mathbf{G}^{(2)}_{:, i_2 , :}  \cdots\mathbf{G}^{(n-1)}_{:,  i_{n-1} , i_n}.
\end{equation}

Similarly to $\tensor{G}^{<n}$, we define $\tensor{G}^{>n}$ as the product of core tensors $\tensor{G}^{(j)}$ for $j\in[n+1,N]$ where $\tensor{G}^{<n}\in\mathbb{R}^{I_1 \times \cdots \times I_{n-1} \times R_{n-1}}$ and $\tensor{G}^{>n}\in\mathbb{R}^{R_n \times I_{n+1} \times \cdots \times I_{N}}$. A tensor network representation of core separation is provided in Figure~\ref{fig:tt_separation}.
\end{defn}

\begin{figure}
\centering
	\scalebox{0.75}[0.75]{\begin{tikzpicture}
	    \begin{scope}[every node/.style={minimum size=3.0em}]
	    \node[draw, shape=circle] (g1) at 			(0,0) {$\tensor{G}^{(1)}$};
	    \node[draw, shape=circle] (g2) at 			(2.0,0) {$\tensor{G}^{(2)}$};
	    \node[draw, shape=circle] (g3) at 			(4.0,0) {$\tensor{G}^{(3)}$};
	    \node[draw, shape=circle] (g4) at 			(6.0,0) {$\tensor{G}^{(4)}$};
	    \node at (3.0,1.5) {\includegraphics[width=0.5cm]{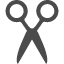}};
	    \node[text width=4cm] at (9.0,0) {{\large $\implies$}};
	    
	     \node[text width=4cm] at (15.0,-1.0) {};
	    \node[draw, shape=circle] (c1) at 			(9.0,0) {$\tensor{G}^{(1)}$};
	    \node[draw, shape=circle] (c2) at 			(11.0,0) {$\tensor{G}^{(2)}$};
	    
	    \node[text width=6cm] at (13.5,1.5) {\large $\tensor{G}^{<3}$};
	     \node at (13.0,1.5) {\includegraphics[width=0.5cm]{sc.png}};
	     \node[text width=6cm] at (18.5,1.5) {\large $\tensor{G}^{>2}$};
	    
	    \node[draw, shape=circle] (c3) at 			(15.0,0) {$\tensor{G}^{(3)}$};
	    \node[draw, shape=circle] (c4) at 			(17.0,0) {$\tensor{G}^{(4)}$};
    
	    \draw (g1) -- (g2) node[midway,above] {$R_1$};
	    \draw (g2) -- (g3) node[pos=0.2,above] {$R_2$};
	    \draw (g3) -- (g4) node[midway,above] {$R_3$};
	    \draw (g1) -- +(0,-1.5cm) node[midway, right = -5pt]{$I_1$};
	    \draw (g2) -- +(0,-1.5cm) node[midway, right = -5pt]{$I_2$};
	    \draw (g3) -- +(0,-1.5cm) node[midway, right = -5pt]{$I_3$};
	    \draw (g4) -- +(0,-1.5cm) node[midway, right = -5pt]{$I_4$};
	    
	    \draw (c1) -- (c2) node[midway,above] {$R_1$};
	    \draw (c2) -- +(1.3cm ,0) node[midway, above]{$R_2$};
	    \draw (c1) -- +(0,-1.5cm) node[midway, right = -5pt]{$I_1$};
	    \draw (c2) -- +(0,-1.5cm) node[midway, right = -5pt]{$I_2$};
	    
	    \draw (c3) -- (c4) node[midway,above] {$R_3$};
	    \draw (c3) -- +(-1.3cm ,0) node[midway, above]{$R_2$};
	    \draw (c3) -- +(0,-1.5cm) node[midway, right = -5pt]{$I_3$};
	    \draw (c4) -- +(0,-1.5cm) node[midway, right = -5pt]{$I_4$};
	    
	   \draw [thick, dashed] (3.0,1.3) -- (3.0,-1.3);
	   \draw [thick, dashed] (13.0,1.3) -- (13.0,-1.3);
	    
	    \end{scope}
	\end{tikzpicture}}

\caption{Visualization of the product of cores given by Tensor Train decomposition of a  tensor $\tensor{X}$ in a space $\mathbb{R}^{I_1 \times I_2 \times I_3 \times I_4}$. The tensor network representations of $\tensor{G}^{>2}$ and $\tensor{G}^{<3}$ are presented.} \label{fig:tt_separation}
\end{figure}

Using Definition~\ref{df:tt_core}, the mode-\emph{n} unfolding of a tensor $\tensor{X} \approx\ \ll \tensor{G}^{(1)}, \cdots , \tensor{G}^{(N)} \gg$ in Eq.~\eqref{eq:tt_def} where $n\in [1,N]$ can be written as
\begin{equation} \label{eq:tt_mtr}
\tensor{X}_{(n)} \approx \tensor{G}^{(n)}_{(2)} \left( \tensor{G}^{>n}_{(1)} \otimes \tensor{G}^{<n}_{(n)} \right)
\end{equation}

\section{Tensor Regression Layer} \label{sc:trl}

\begin{figure}
\centering
	\scalebox{0.7}[0.7]{\begin{tikzpicture}
	    \begin{scope}[every node/.style={minimum size=3.0em}]
	    \node[text width=4cm] at (0,0) {{\large $f(\tensor{X}) = $}};
	    \node[draw, shape=circle] (x) at 		(0, 0) {$\tensor{X}$};
	    \node[draw, shape=circle] (w) at 		(2.0,0) {$\tensor{W}$};

    	\draw [-, out=90, in=90, bend left=55] (x.north) to node[above]{$I_1$}(w.north);
        \draw [-, out=270, in=270, bend right=55] (x.south) to node[above]{$I_3$}(w.south);
	    \draw (x) -- (w) 					node[midway, above] {$I_2$};
	    \draw (w) -- +(1.3cm,0) 			node[midway, above] {$I_4$};
	    \end{scope}
	\end{tikzpicture}}
	
\caption{Visualization of tensor regression layer (TRL) using tensor networks. $\tensor{X}\in \mathbb{R}^{I_1\times I_2 \times I_3}$ and a weight tensor $\tensor{W}\in \mathbb{R}^{I_1\times I_2 \times I_3 \times I_4}$ are represented by circular nodes connected by edges which represents contraction operation between two tensors.} \label{fig:reg}
\end{figure}
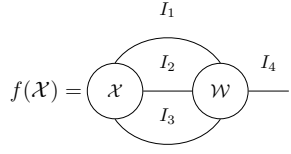

In this section, we introduce tensor regression layer via various low-rank tensor approximations. As stated in Section~\ref{sc:intro}, the last fully-connected layer of traditional CNN represents a large proportion of the model parameters. In addition to such large consumption of computational resources, the flattening operation leads to the loss of rich multi-modal information in the last convolutional layer. Tensor regression layer \cite{kossaifi2017tensor} replaces such last flattening and fully connected layers of CNN by a multilinear map with low Tucker rank. In this work, we explore imposing other low-rank constraints on the weight tensor and we compare the compression and regularization effects of using either CP, Tucker or TT decompositions.

Given an input tensor $\tensor{X}\in \mathbb{R}^{I_1\times \cdots \times I_N}$ and a weight tensor $\tensor{W}\in \mathbb{R}^{I_1\times \cdots \times I_N \times I_{N+1}}$, we investigate the function $f: \mathbb{R}^{ I_1\times \cdots \times I_N}\rightarrow \mathbb{R}^{I_{N+1}}$ where $I_{N+1}$ is the number of classes. Given such two tensors, the function $f$ is defined as

\begin{equation} \label{eq:formulation}
f(\tensor{X}) = \tensor{W}_{(N+1)} vec(\tensor{X}) + \mathbf{b}
\end{equation}
where $\mathbf{b}\in \mathbb{R}^{I_{N+1}}$ is a bias vector added to the product of $\tensor{W}$ and $\tensor{X}$. The tensor network representation of an example of Eq.~\eqref{eq:formulation} is given in Figure~\ref{fig:reg}. 
The main idea behind tensor tensor regression layers is to enforce a low tensor rank structure on $\tensor{W}$ in order to both reduce memory usage and to 
leverage the multilinear structure of the input $\tensor{X}$.

Throughout the paper, we denote a TRL with TT decomposition by TT-TRL. Similarly we use CP-TRL and Tucker-TRL for a TRL with CP or Tucker decomposition.

{\bf CP decomposition.} \hspace{5pt} First we investigate applying CP decomposition to approximate the weight tensor $\tensor{W}$. Using Eq.~\eqref{eq:cp_def} and Eq.\eqref{eq:cp_mtr}, Eq.~\eqref{eq:formulation} can be rewritten as

\begin{equation} \label{eq:formulation_cp}
\begin{split}
f(\tensor{X}) 	&\approx \left( \jump{\mathbf{A}^{(1)}, \mathbf{A}^{(2)}, \cdots , \mathbf{A}^{(N)}, \mathbf{A}^{(N+1)} } \right)_{(N+1)} vec(\tensor{X})  + \mathbf{b}\\
			&= \mathbf{A}^{(N+1)}\left( \mathbf{A}^{(N)} \odot \cdots \odot \mathbf{A}^{(1)} \right)^T  vec(\tensor{X}) + \mathbf{b}
\end{split}
\end{equation}
We can use this formulation to obtain the partial derivatives needed to implement gradient based optimization methods~(\eg backpropagation), indeed
\begin{equation} \label{eq:gradient_cp}
\frac{\partial f(\tensor{X})_i}{\partial (\mathbf{A}^{(n)})_{jk}} = \frac{(\partial \mathbf{A}^{(N+1)}\left( \mathbf{A}^{(N)} \odot \cdots \odot \mathbf{A}^{(1)} \right)^T vec(\tensor{X}) )_{i}}{\partial (\mathbf{A}^{(n)})_{jk}}
\end{equation} 
for all of the matrices $\mathbf{A}^{(n)}$ for $n \in [1,N+1]$. Furthermore, for a given mode $n$, we can naturally arrange these partial derivatives into a third order tensor 
$\partial f(\tensor{X}) / \partial \mathbf{A}^{(n)} \in \mathbb{R}^{I_{N+1} \times I_{n}  \times R}$ and obtain their expression using unfolding: 

$$ 
\left(\frac{\partial f}{\partial \mathbf{A}^{(n)}}\right)_{(2)} = (\tensor{X})_{(n)} (\mathbf{A}^{(N)} \odot \cdots \odot \mathbf{A}^{(n+1)}  \odot \mathbf{A}^{(n-1)}
\odot \cdots \odot \mathbf{A}^{(1)})  (\mathbf{A}^{(N+1)} \odot \mathbf{I}_R)^T
$$
for $n\in[1,N]$, and
$$
\left(\frac{\partial f}{\partial \mathbf{A}^{(N+1)}}\right)_{(1)} = \mathbf{I}_{I_{N+1}} \otimes (vec({\tensor{X}})^T (\mathbf{A}^{(N)} \odot \cdots \odot \mathbf{A}^{(1)})). 
$$

{\bf Tucker decomposition.} \hspace{5pt} As described in Section~\ref{sc:back}, the Tucker decomposition approximates an input tensor by a core tensor and a set of factor matrices. We can rewrite Eq.~\eqref{eq:formulation} using approximation of the tensor $\tensor{W}$ by Tucker decomposition as 

\begin{equation} \label{eq:formulation_t}
f(\tensor{X})  \approx \mathbf{U}^{(N+1)} \tensor{G}_{(N+1)} \left( \mathbf{U}^{(N)} \otimes \cdots \otimes \mathbf{U}^{(1)} \right)^T  vec(\tensor{X}) + \mathbf{b}
\end{equation}
where the tensor $\tensor{W}$ is approximated with
\begin{equation} \label{eq:formulation_t_decomp}
\tensor{W} \approx \tensor{G} \times_1 \mathbf{U}^{(1)} \times_2 \cdots \times_N \mathbf{U}^{(N)} \times_{N+1} \mathbf{U}^{(N+1)} = \jump{\tensor{G}; \mathbf{U}^{(1)}, \cdots , \mathbf{U}^{(N+1)}}
\end{equation}
The tensor network representations of Eq.~\eqref{eq:formulation_t} is shown in Figure~\ref{fig:tt_t}. Given a tensor of size $\mathbb{R}^{I_1 \times \cdots \times I_N}$, the function $f$ maps such tensor to the space $\mathbb{R}^{I_{N+1}}$ with low-rank constraints.

We can again  obtain concise expressions for the partial derivatives using unfoldings, for example:
\begin{equation} 
\left(\frac{\partial f(\tensor{X})}{\partial \mathbf{U}^{(1)}}\right)_{(1)} = \left(\jump{\tensor{G}; \mathbf{I}_{R_{1}}, \mathbf{U}^{(2)}, \cdots, \mathbf{U}^{(N+1)}} \right)_{(N+1)} \left(\tensor{X}_{(1)} \otimes \mathbf{I}_{R_1}\right)^T,
\end{equation}
\begin{equation} 
\left(\frac{\partial f(\tensor{X})}{\partial \mathbf{U}^{(N+1)}}\right)_{(1)} = \left(\left(\jump{\tensor{G};  \mathbf{U}^{(1)}, \cdots, \mathbf{U}^{(N)}, \mathbf{I}_{R_{N+1}}} \right)_{(N+1)} vec(\tensor{X})\right)^T \otimes \mathbf{I}_{I_{N+1}}
\end{equation}
and
\begin{equation} 
\left(\frac{\partial f(\tensor{X})}{\partial \tensor{G}}\right)_{(1)} = \mathbf{U}^{(N+1)} \otimes  vec( \jump{\tensor{X};  (\mathbf{U}^{(1)})^T, \cdots, (\mathbf{U}^{(N)})^T} )^T.
\end{equation}

{\bf Tensor Train decomposition.} \hspace{5pt} The tensor network visualization is given in Figure~\ref{fig:tt_t}, where the weight tensor $\tensor{W}$  is replaced with its TT representation.
Using Eq.~\eqref{eq:tt_def} and~\eqref{eq:tt_mtr}, in the case of TT decomposition Eq.~\eqref{eq:formulation} can be rewritten as

\begin{equation} \label{eq:formulation_tt}
\begin{split}
f(\tensor{X}) &\approx \left(\tensor{G}^{(N+1)}\right)_{(2)} \left( \left(\tensor{G}^{>N+1}\right)_{(1)} \otimes \left(\tensor{G}^{<N+1}\right)_{(N+1)} \right)  vec(\tensor{X}) + \mathbf{b}\\
&= \left(\tensor{G}^{(N+1)}\right)_{(2)} \left(\tensor{G}^{<N+1}\right)_{(N+1)}  vec(\tensor{X}) + \mathbf{b}
\end{split}
\end{equation}
where the second equality follows from the fact that $\tensor{G}^{>N+1}=1$. Similarly to the case of CP and Tucker decomposition, the 
partial derivatives can be summarized with
\begin{equation} 
\left(\frac{\partial f(\tensor{X})_i}{\partial \tensor{G}^{(n)}}\right)_{(2)}
 = 
 \tensor{X}_{(n)} \left(\left(\tensor{G}^{>n}_{:,\cdots,:,i}\right)_{(1)}\otimes \left(\tensor{G}^{<n}\right)_{(n)}\right)^T
\end{equation}
for all $i\in[1,I_{N+1}]$ and $n\in[1,N]$, and
\begin{equation} 
\left(\frac{\partial f(\tensor{X})}{\partial \tensor{G}^{(N+1)}}\right)_{(1)}^T
 = 
\left(\left(\tensor{G}^{<N+1}\right)_{(N+1)} vec(\tensor{X})\right)\otimes \mathbf{I}_{I_{N+1}}.
\end{equation}

\begin{figure}
\centering

	\subfloat[TT-TRL]{
	\scalebox{0.65}[0.65]{\begin{tikzpicture}
	    \begin{scope}[every node/.style={minimum size=3.4em}]
	    \node[draw, shape=circle] (x) at 		(0, 0) {$\tensor{X}$};
	    \node[draw, shape=circle] (g1) at 		(2.0,0) {$\tensor{G}^{(1)}$};
	    \node[draw, shape=circle] (g2) at 		(2.0,-1.5) {$\tensor{G}^{(2)}$};
	    \node(uabbv) at 					(2.0,-3.0) {$\cdots$};
	    \node[draw, shape=circle] (gnprev) at 	(2.0,-4.5) {$\tensor{G}^{(N)}$};
	    \node[draw, shape=circle] (gn) at 		(4.0,-4.5) {$\tensor{G}^{(N+1)}$};

	    \draw (x) -- (g1) node[midway,above] {$I_1$};
	    \draw (x) -- (g2);
	    \draw (x) -- (uabbv);
	    \draw (x) -- (gnprev) node[midway, left]{$I_N$};
	    
	    \draw (g1) -- (g2) node[midway, right]{$R_1$};	
	    \draw (g2) -- (uabbv);
	    \draw (uabbv) -- (gnprev);
	    \draw (gn) -- (gnprev) node[midway, above]{$R_{N}$};
	    \draw (gn) -- +(1.5cm,0) node[midway, above]{$I_{N+1}$};
	    \end{scope}
	\end{tikzpicture}\label{fig:tn_tt}}
	} \qquad \subfloat[Tucker-TRL]{
	\scalebox{0.65}[0.65]{\begin{tikzpicture}
	    \begin{scope}[every node/.style={minimum size=3.4em}]
	    \node[draw, shape=circle] (x) at 		(0, -2.25) {$\tensor{X}$};
	    \node[draw, shape=circle] (u1) at 		(2.0,0) {$\mathbf{U}^{(1)}$};
	    \node[draw, shape=circle] (u2) at 		(2.0,-1.5) {$\mathbf{U}^{(2)}$};
	    \node(uabbv) at 					(2.0,-3.0) {$\cdots$};
	    \node[draw, shape=circle] (u3) at 		(2.0,-4.5) {$\mathbf{U}^{(N)}$};
	    \node[draw, shape=circle] (g) at 		(4.0,-2.25) {$\tensor{G}$};
	    \node[draw, shape=circle] (u4) at 		(6.0,-2.25) {$\mathbf{U}^{\tiny (N+1)}$};

        \draw [-, out=90, in=180, bend left=45] (x.north) to node[above]{$I_1$}(u1.west);
	    \draw [-, out=45, in=180, bend left=15] (x.north east) to node[above]{$I_2$}(u2.west);
        \draw [-, out=270, in=180, bend right=45] (x.south) to node[below]{$I_N$}(u3.west);
	    \draw [-, out=-45, in=180, bend right=15] (x.south east) to (uabbv.west);

        \draw [-, out=0, in=90, bend left=45] (u1.east) to node[above]{$R_1$}(g.north);
	    \draw [-, out=0, in=135, bend left=15] (u2.east) to node[above]{$R_2$}(g.north west);
        \draw [-, out=0, in=270, bend right=45] (u3.east) to node[below]{$R_N$}(g.south);
	    \draw (u4) -- (g) node[midway, above]{$R_{N+1}$};
	    \draw [-, out=0, in=225, bend right=15] (uabbv.east) to (g.south west);
	    \draw (u4) -- +(1.5cm,0) node[midway, above]{$I_{N+1}$};
	    
	    \end{scope}
	\end{tikzpicture}}
	\label{fig:tn_tucker}}

\caption{Tensor Network visualization of tensor regression layer via Tucker and TT decompositions. Each label attached to corresponding edge represents the dimension shared between two tensors by tensor contraction operation.} \label{fig:tt_t}
\end{figure}
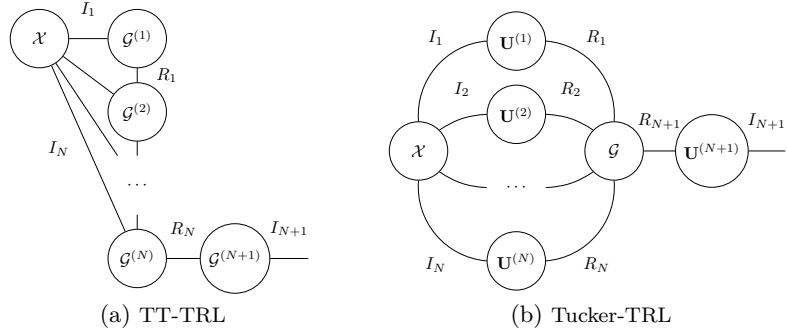

\section{Tensor perspective on Global Average Pooling layer} \label{sc:gap}

In this section, we provide an insight on Global Average Pooling layer from the perspective of tensor algebra. In particular, we show that GAP layer is a special case of Tucker-TRL.

It is a traditional practice to apply flattening operation to the output tensor (i.e. the last convolutional layer) before extracting its features. The problem of such approach lies in the generalization ability to the test dataset. Some work on deep neural networks show that fully-connected layers are prone to overfitting, thus leading to poor performance on test dataset \cite{hinton2012improving, lin2013network, krizhevsky2012imagenet}.

In order to tackle such generalizability problem and to provide regularization, Global Average Pooling (GAP) layer was presented by Lin et al. \cite{lin2013network}. It replaces the combination of vectorization operation and fully-connected layer with averaging operation over all slices along the output channel axis. The output of a GAP layer is thus a single vector of size the number of output channels. GAP layer was empirically shown to significantly reduce the number of model parameters in CNNs~\cite{lin2013network}.

The authors of~\cite{lin2013network} claims not only that GAP layer reduces the trainable model parameters but also that it can prevent the model from overfitting during the training stage. Over the last decade, GAP layer has been adopted to some of the most successful image classification models such as Residual Networks and VGG-16 \cite{he2016deep, simonyan2014very}.

More general interpretation of the convolutional output is that it is a high-order tensor $\tensor{X}$ in a space $\mathbb{R}^{I_1 \times \cdots \times I_N}$. Given such tensor, the GAP layer will output a vector $\mathbf{y}$ defined by

\begin{equation} \label{eq:gap_traditional}
\left( \mathbf{y} \right)_{i_N} = \prod\limits_{n=1}^{N-1} I^{-1}_n \sum\limits_{i_1=1}^{I_1} \sum\limits_{i_2 = 1}^{I_2} \cdots \sum\limits_{i_{N-1} = 1}^{I_{N-1}} \tensor{X}_{i_1 i_2 \dots i_N}\ \ \ 
\text{for all } i_N\in[1,I_N].
\end{equation}

We here assume that the axis for the output channel corresponds to the last mode of the tensor $\tensor{X}$.
We now show that a GAP layer mapping  $\tensor{X}$ to $\mathbf{y}\in\mathbb{R}^{I_N}$ is equivalent to a specific Tucker-TRL with rank $(1,\cdots,1,I_N,I_N)$. Indeed,
let 
$$\tensor{W} = \jump{\tensor{G}; \mathbf{u}^{(1)},\cdots, \mathbf{u}^{(N-1)}, \mathbf{U}^{(N)}, \mathbf{U}^{(N+1)}  } \in \mathbb{R}^{I_1\times\cdots\times I_{N} \times I_N}$$
 be the
regression tensor of a Tucker-TRL, with $(\mathbf{u}^{(n)})_{i_n} = 1/I_n$ where $i_n \in [1,I_n]$ for each $n\in[1,N-1]$, and $\mathbf{U}^{(N)} = \mathbf{U}^{(N+1)} = \tensor{G} = \mathbf{I}_{I_N}$.
We have
\begin{equation} \label{eq:gap_equivalent_trn}
\begin{split}
f(\tensor{X})_{i_N} & = 
 \left(\mathbf{U}^{(N+1)} \tensor{G}_{(N+1)} \left(\mathbf{U}^{(N)} \otimes \mathbf{u}^{(N-1)} \otimes \cdots \otimes \mathbf{u}^{(1)} \right)^T  vec(\tensor{X})\right)_{i_N}\\
&=
 (\tensor{X}\sbullet[.75]_1 \mathbf{u}^{(1)} \sbullet[.75]_{2} \cdots \sbullet[.75]_{N-1} \mathbf{u}^{(N-1)})_{i_N}\\
                    & = \sum\limits_{i_{N-1}=1}^{I_{N-1}} \cdots \sum\limits_{i_{2}=1}^{I_{2}}  \sum\limits_{i_{1}=1}^{I_{1}} \tensor{X}_{i_1 , i_2 , \dots , i_N} (\mathbf{u}^{(1)})_{i_1} (\mathbf{u}^{(2)})_{i_2} \cdots (\mathbf{u}^{(N-1)})_{i_{N-1}} \\
                    & =  \prod\limits_{n=1}^{N-1} I_n^{-1} \sum\limits_{i_1=1}^{I_1} \sum\limits_{i_2 = 1}^{I_2} \cdots \sum\limits_{i_{N-1} = 1}^{I_{N-1}} \tensor{X}_{i_1 , i_2 , \dots  , i_N}\\
&=
(\mathbf{y})_{i_N}.
\end{split}
\end{equation}
Observe that the composition of a GAP layer with a fully connected layer mapping $\mathbf{y}\in \mathbb{R}^{I_N}$ to $\mathbf{z}\in \mathbb{R}^O$ can also be achieved using a unique Tucker-TRL
by setting  $\mathbf{U}^{(N+1)}\in\mathbb{R}^{I_N\times O}$ to be the weight matrix of the fully connected layer instead of the identity. A graphical representation of this equivalence is shown in Figure~\ref{fig:tensor_network_gap}.

\begin{figure}
\centering
	
    \subfloat[Substitute matrix to vector]{
 	\scalebox{0.55}[0.55]{
    \begin{tikzpicture}
	    \begin{scope}[every node/.style={minimum size=3.4em}]
	    \node[draw, shape=circle] (x) at 		(0, -2.25) {$\tensor{X}$};
	    \node[draw, shape=circle] (u1) at 		(2.0,0) {$\mathbf{u}^{(1)}$};
	    \node[draw, shape=circle] (u2) at 		(2.0,-1.5) {$\mathbf{u}^{(2)}$};
	    \node(uabbv) at 					(2.0,-3.0) {$\cdots$};
	    \node[draw, shape=circle] (u3) at 		(2.0,-4.5) {$\mathbf{U}^{(N)}$};
	    \node[draw, shape=circle] (g) at 		(4.0,-2.25) {$\tensor{G}$};
	    \node[draw, shape=circle] (u4) at 		(6.0,-2.25) {$\mathbf{U}^{\tiny (N+1)}$};
    
        \draw [-, out=90, in=180, bend left=45] (x.north) to node[above]{$I_1$}(u1.west);
	    \draw [-, out=45, in=180, bend left=15] (x.north east) to node[above]{$I_2$}(u2.west);
        \draw [-, out=270, in=180, bend right=45] (x.south) to node[below]{$I_N$}(u3.west);
	    \draw [-, out=-45, in=180, bend right=15] (x.south east) to (uabbv.west);

        \draw [-, out=0, in=270, bend right=45] (u3.east) to node[below]{$R_N$}(g.south);
	    \draw (u4) -- (g) node[midway, above]{$R_{N+1}$};
	    \draw (u4) -- +(1.5cm,0) node[midway, above]{$I_{N+1}$};
	    
	    \end{scope}
	\end{tikzpicture}}\label{fig:trn_gap_step1}} \subfloat[Average along the axis of output channel]{\scalebox{0.55}[0.55]{
    \begin{tikzpicture}
	    \begin{scope}[every node/.style={minimum size=3.4em}]

		\node[draw, shape=circle] (x) at 		(0,-2.25) {$\mathbf{y}$};
        \node[draw, shape=circle] (u_n) at 		(2,-3.) {$\mathbf{U}^{(N)}$};
        \node[draw, shape=circle] (g) at 		(4,-3.) {$\mathbf{G}$};
        \node[draw, shape=circle] (u_last) at 		(6,-3.) {$\mathbf{U}^{(N+1)}$};
        
        \node(nothing) at 					(6.0,-4.5) {$$};
        
        \draw (x) -- (u_n) node[midway, above]{$I_{N}$};
        \draw (u_n) -- (g) node[midway, above]{$R_{N}$};
        \draw (g) -- (u_last) node[midway, above]{$R_{N+1}$};
        \draw (u_last) -- +(1.5cm,0) node[midway, above]{$I_{N+1}$};
        
	    \end{scope}
	\end{tikzpicture}}\label{fig:trn_gap_step2}} \subfloat[Linear Transformation]{
    \scalebox{0.55}[0.55]{
    \begin{tikzpicture}
	    \begin{scope}[every node/.style={minimum size=3.4em}]

		\node[draw, shape=circle] (x) at 		(0,-2.25) {$\mathbf{y}$};
        \node[draw, shape=circle] (u_last) at 		(2,-3.) {$\mathbf{U}^{(N+1)}$};
        
        \node(nothing) at 					(2.0,-4.5) {$$};
        
        \draw (x) -- (u_last) node[midway, above]{$I_{N}$};
        \draw (u_last) -- +(1.5cm,0) node[midway, above]{$I_{N+1}$};
        
	    \end{scope}
	\end{tikzpicture}}\label{fig:trn_gap_step3}}

\caption{Tensor network representation of GAP layer. \protect\subref{fig:trn_gap_step1} factor matrices $\mathbf{U}^{(n)}$ are replaced with vectors $\mathbf{u}^{(n)}$ for $i\in[1,N-1]$. \protect\subref{fig:trn_gap_step2} contraction operation between $\tensor{X}$ and factor matrices are performed. \protect\subref{fig:trn_gap_step3} most simplified version; the product of a matrix and a vector.} \label{fig:tensor_network_gap}
\end{figure}
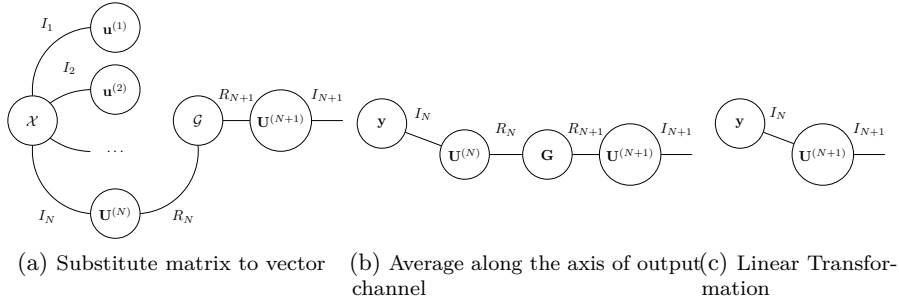

\section{Observations on Rank Constraints} \label{sc:observation_on_rank_constraints}
In this section, we provide a simple guideline for choosing one of the components of low-rank constraints enforced to TRL. In particular, we observe that the CP rank parameter and and the last Tucker/TT rank parameter  affects the dimension of the image of the function computed by the TRL. For example, as a consequence of this observation, if a TRL is used as the last layer of a network before a softmax activation function in a classification tasks with $O$ classes, setting the rank parameter to values greater than $O$ leads to unnecessary redundancy, while setting it to smaller values can detrimentally limit the expressiveness of the resulting classifier.

First, we start with a simple lemma necessary to provide the upper-bound on the dimension of the image of the regression function. We show that if an input matrix admits a factorization, then a function $f$ which maps such matrix to a linear space has an upper-bound on the dimension of the image.

\begin{lem} \label{lem:vec}
If $f: \mathbf{v} \mapsto \mathbf{A}\mathbf{B}\mathbf{v}$ with $ \mathbf{A}\in \mathbb{R}^{m\times R}$ and $ \mathbf{B}\in \mathbb{R}^{R \times n}$, then $dim\left(Im\left(f\right)\right) \leq R$ where $f: \mathbb{R}^{n} \rightarrow \mathbb{R}^{m}$.
\end{lem}
\begin{proof}
Given such function f, the dimension of the image of the function f is $dim\left(Im\left(f\right)\right) = dim\left(span\left( \{ f\left( \mathbf{v} \right) \mid \mathbf{v}\in\mathbb{R}^{n} \} \right) \right)$, which is the dimension of the space that is spanned by column vectors of $\mathbf{X}$ = $\mathbf{A}\mathbf{B}$. That is, $dim\left(span\left( \{ f\left( \mathbf{v} \right) \mid \mathbf{v}\in\mathbb{R}^{n} \} \right) \right) = dim\left(span \left( \{ \mathbf{x}_1, \cdots , \mathbf{x}_n \} \right) \right)$. It is clear that each column vector of the matrix $\mathbf{X}=\mathbf{A}\mathbf{B}$ is linear combination of column vectors of $\mathbf{A}$ from the equation $(\mathbf{A}\mathbf{B})_{:i} = \mathbf{A}\mathbf{b}_i=\sum_{r=1}^{R} \mathbf{a}_r (\mathbf{b}_i)_r$ where $\mathbf{b}_i$ denotes $i$-th column vector of $\mathbf{B}$. Since matrix $\mathbf{A}$ is in the space $\mathbb{R}^{m\times R}$, the dimension of the span of the column vectors of $\mathbf{X}$ is upper-bounded by $R$, namely $dim\left(Im\left(f\right)\right) = dim\left(span \left( \{ \mathbf{x}_1, \cdots , \mathbf{x}_n \} \right) \right) \leq rank(\mathbf{A})\leq R$.
\end{proof}

Using Lemma \ref{lem:vec}, we can provide upper-bounds on the dimension of the image spanned by the regression function of a TRL for different tensor rank constraints. 

\begin{prop} \label{prop:bottleneck}
Let $f: \mathbb{R}^{I_1\times \cdots \times I_N} \rightarrow \mathbb{R}^{I_{N+1}}$ where $f: \tensor{X} \mapsto \tensor{W}_{(N+1)}  vec(\tensor{X})$. The following hold:
\begin{itemize}
\item if $\tensor{W}$ admits a  TT decomposition of rank $(1,R_1, \cdots ,R_{N},1)$, then $dim(Im(f)) \leq R_N$,
\item if $\tensor{W}$ admits a Tucker decomposition of  rank $(R_1,R_2, \cdots ,R_{N+1})$, then  $dim(Im(f)) \leq R_{N+1}$,
\item if $\tensor{W}$ admits a CP decomposition of  rank $R$, then  $dim(Im(f)) \leq R$.
\end{itemize}
\end{prop}
\begin{proof}
If $\tensor{W}$ admits a TT decomposition with TT rank $\{1,R_1, \cdots ,R_{N},1\}$, then by Eq.~\eqref{eq:tt_def}, we have $\tensor{W}_{i_{1} , i_{2} , \ldots , i_{N+1}} = \mathbf{G}^{(1)}_{i_1  , :} \times \mathbf{G}^{(2)}_{:, i_2 , :} \times \cdots \times \mathbf{G}^{(N+1)}_{: , i_{N+1}}$. Using the matricization of $\tensor{W}$ given by Eq.~\eqref{eq:tt_mtr}, we can write $f$ as follows;
\begin{equation}
\begin{split}
f(\tensor{X}) 	&= \tensor{W}_{(N+1)} \times vec(\tensor{X})\\
			&= \tensor{G}^{(N+1)}_{(2)} \tensor{G}^{<N+1}_{(N+1)} \times vec(\tensor{X})
\end{split}
\end{equation}
and consequently, since $\tensor{G}^{(N+1)}_{(2)}$ and  $\tensor{G}^{<N+1}_{(N+1)}$ are of size
$\mathbb{R}^{I_{N+1} \times R_N}$ and $\mathbb{R}^{R_N \times I_1 I_2 \dots I_N}$ respectively,  we have 
$dim(Im(f)) \leq R_N$ by Lemma \ref{lem:vec}.

The other two points can be proven in a similar fashion using  Eq.~\eqref{eq:tc_def} and~\eqref{eq:tc_mtr} for Tucker, and Eq.~\eqref{eq:cp_def} and~\eqref{eq:cp_mtr} for CP.
\end{proof}

We have shown that the dimension of the image mapped by the function $f$ is upper-bounded by one of the tensor rank parameters. We refer to this specific component of the rank tuple as the \emph{bottleneck rank}.

\begin{defn} \label{df:bottleneck_rank}
Given a regression tensor $\tensor{W}\in\mathbb{R}^{I_1 \times \cdots \times I_N}$, if $\tensor{W}$ admits a Tucker Decomposition with rank $(R_1,R_2, \cdots ,R_{N+1})$, we define the rank $R_{N+1}$ as the \emph{bottleneck rank}. Similarly, if $\tensor{W}$ admits a TT decomposition with \emph{TT-rank} $(1,R_1, \cdots ,R_{N},1)$, we define $R_N$ as the \emph{bottleneck rank}.
\end{defn}

This observation on the rank constraints used in a tensor regression layer can provide a simple guideline for choosing the bottleneck rank. For instance, when a TRL is used as the last layer of an architecture for a classification task, setting the \emph{bottleneck rank} to a value smaller than the number of classes $O$  could limit the expressiveness of TRL~(which we will empirically demonstrate in Section~\ref{sc:exp_mnist}), while setting it to a value higher than $O$ could lead to redundancy in the model parameters.

\section{Experiments} \label{sc:exp}
In this section we provide experimental evidence which 1) supports our analysis on TRLs in Section~\ref{sc:observation_on_rank_constraints} and 2) investigate the compressive and regularization power of the different low-rank constraints. We present experiments with tensor regression layer using CP, Tucker and TT decomposition on the benchmark datasets  MNIST~\cite{lecun1998mnist}, FashionMNIST~\cite{xiao2017fashion}, CIFAR-10 and CIFAR-100~\cite{krizhevsky2009learning}.

\begin{figure}
\resizebox{\textwidth}{!}{
  \centering
  \includegraphics[]{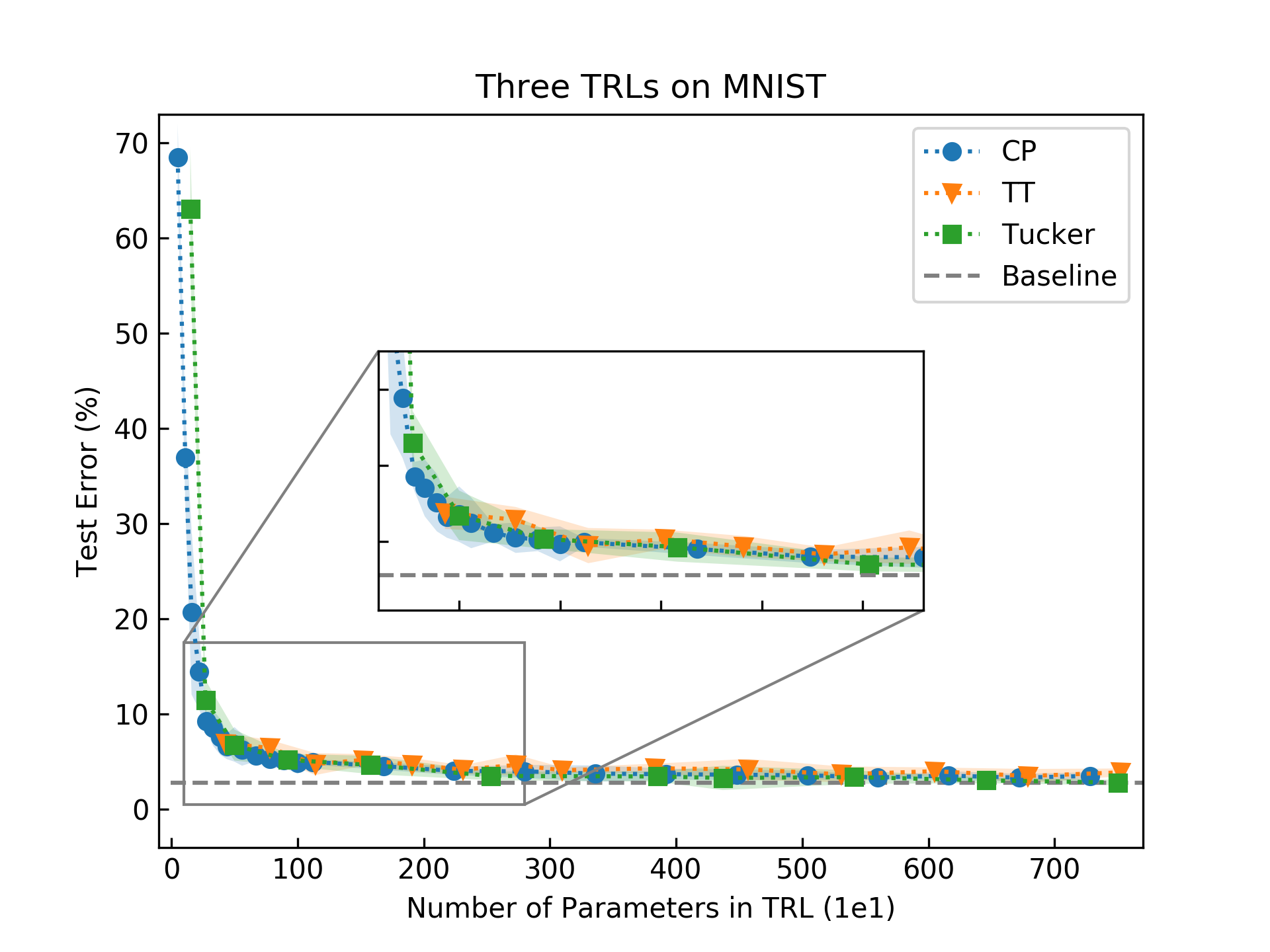}\hfill
  \includegraphics[]{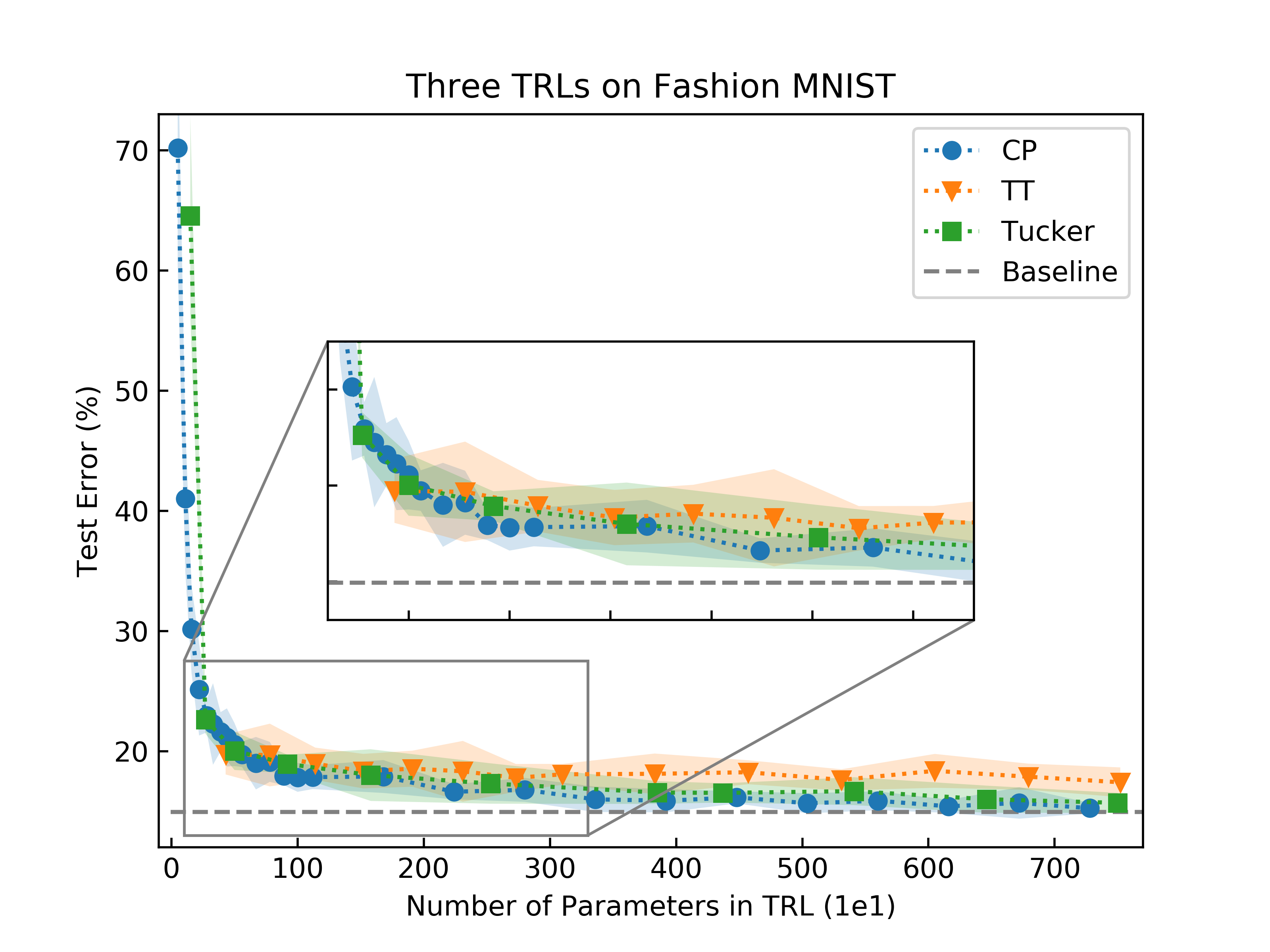}
  
  \caption{Best viewed in color. Test error as a function of the number of parameters in TRL. Performances of three types of TRL (CP, TT and Tucker) are compared in terms of regularization effect of TRL. \emph{Left:} MNIST dataset. \emph{Right:} Fashion-MNIST dataset. For all entries, we run experiments for 5 times and presented with confidence interval with critical value of $95\%$.}  \label{fig:mnist}}
\end{figure}

\subsection{MNIST and Fashion-MNIST dataset} \label{sc:exp_mnist}

MNIST dataset \cite{lecun1998mnist} consists of $28 \times 28$ 1-channel images of handwritten digits from $0$ to $9$. The dataset contains $60$k training and a test set of $10$k examples. The purpose of the experiment is to provide insights on regularization power of different low-rank constraints. We set our baseline classifier to be CNN with $2$ convolutional layers followed by $1$ fully-connected layer. Rectified linear units (ReLU) \cite{nair2010rectified} were introduced between each layer as non-linear activations. We tested the model with three tensor approximations; CP, Tucker and TT. By applying various low-rank constraints, we aim to show that as such constraints become larger, the smaller the approximation error becomes, therefore the accuracy of the low-rank model approaches to that of the model without regularizations (\emph{i.e.} low-rank constraints).

We concisely review the choice of low-rank constraints for Tucker, TT and CP models. Detailed experimental configuration is available 
online\footnote{\url{https://github.com/xwcao/LowRankTRN}}. Given an output tensor from final convolutional layer $\tensor{X}\in\mathbb{R}^{S \times 7 \times 7 \times 32}$ where $S$ denotes the number of samples in one batch, we constrain the weight tensor $\tensor{W}$ with the rank of Tucker decomposition $\{ R_i \}_{i=1}^{4}$. Following Proposition~\ref{prop:bottleneck}, the \emph{bottleneck rank} is set to $10$ for TRL with Tucker and TT constraints. 

Following~\cite{krizhevsky2012imagenet}, we initialized the weights in each layer from zero mean normal distribution with standard deviation $\sigma = 0.1$. The bias term of each layer is initialized with constant $0.1$. For Tucker-TRL, we conducted a total of $32$ experiments. This is per each low-rank Tucker-TRL where $R_i$ were set with constraints $R_1\leq7$, $R_2\leq7$ and $R_3\leq32$ respectively. A set of experiments were conducted for TT-TRL as well. We set \emph{TT-rank} to be $R_2\leq7$ and $R_3\leq32$. For CP-TRL, we simply evaluated the performance with rank $r$ from a set $[1,100]$.

We evaluate empirical performance of TRL with another MNIST-like dataset: Fashion-MNIST. The dataset consists of $60$k training and $5$k testing images where each sample belongs to one of ten classes of fashion items such as shoes, clothes and hats. We used the same CNN architecture and hyperparamters as for the MNIST dataset.

Experimental outcomes for both datasets are provided in Figure~\ref{fig:mnist} where we can see that all low-rank approximation models exhibit similar performance in both MNIST and Fashion-MNIST dataset. As for the regularization effect, however, it is observable that as we relax the low-rank constraints, the accuracies of each model gradually converge to that of baseline model. This result illustrates the effect of regularization power that low-rank constraints provide. We also conducted experiments where we used GAP layer instead of fully-connected layer on both MNIST and Fashion-MNIST dataset. In both cases, the model performed very poorly compared to that of fully connected layer; $71.17\%$ with MNIST dataset and $68.40\%$ with Fashion-MNIST.

We conducted similar experiment to provide a empirical support to Proposition~\ref{prop:bottleneck}. In Section~\ref{sc:observation_on_rank_constraints} we showed that the dimension of the image of TRL is upper-bounded by the \emph{bottleneck rank}. We conducted experiments where we fix the \emph{bottleneck rank} to be one of $\{1,2,5,10\}$. The experimental result presented in Figure~\ref{fig:mnist_comaprison} shows the clear distinctions among models with different \emph{bottleneck ranks}. It is observable that \emph{bottleneck rank} affects the test accuracy by providing upper-bound to the dimension of the image of TRL.

\begin{figure}
  \centering
  \includegraphics[width=0.5\textwidth]{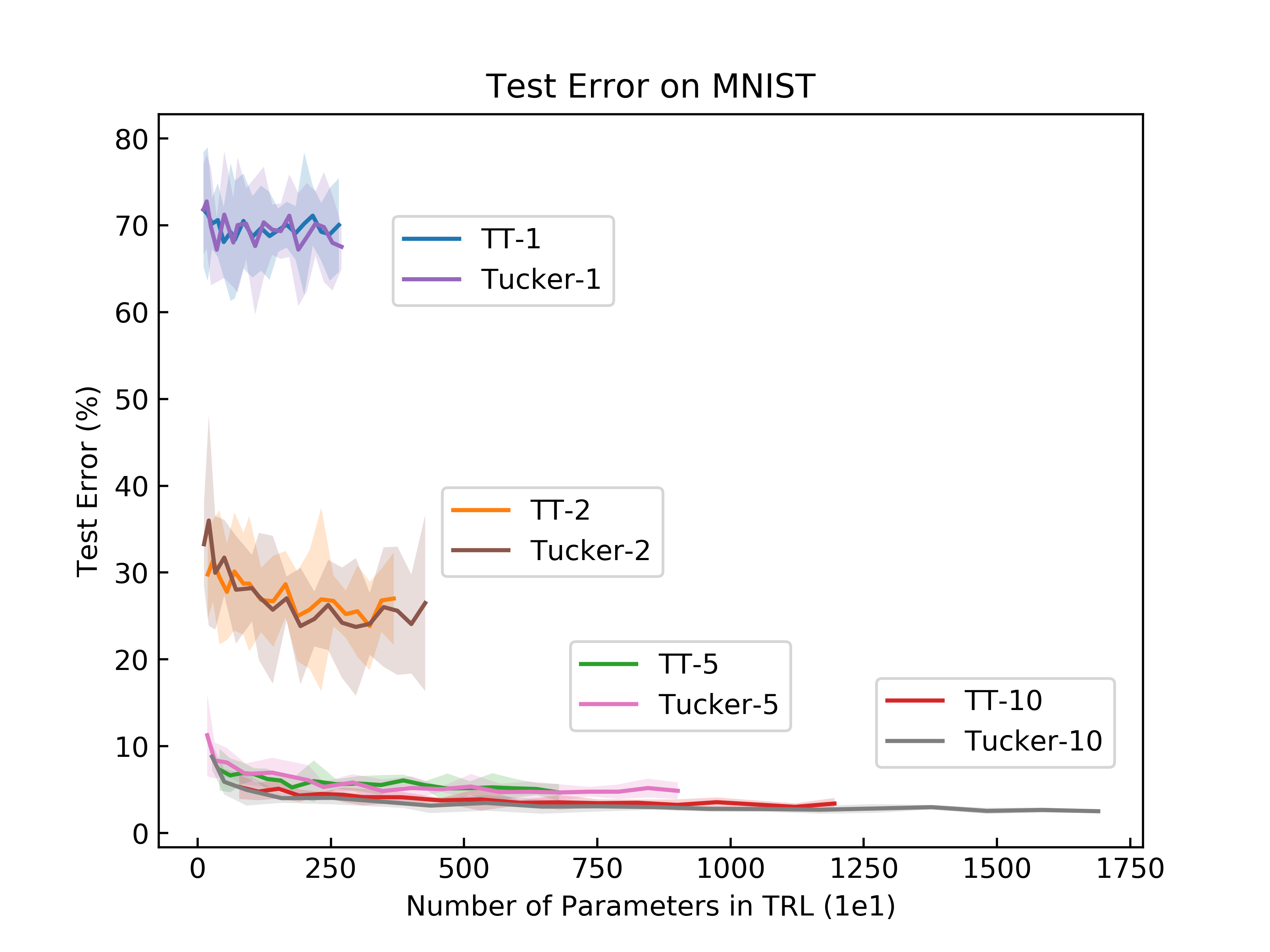}\hfill
  \includegraphics[width=0.5\textwidth]{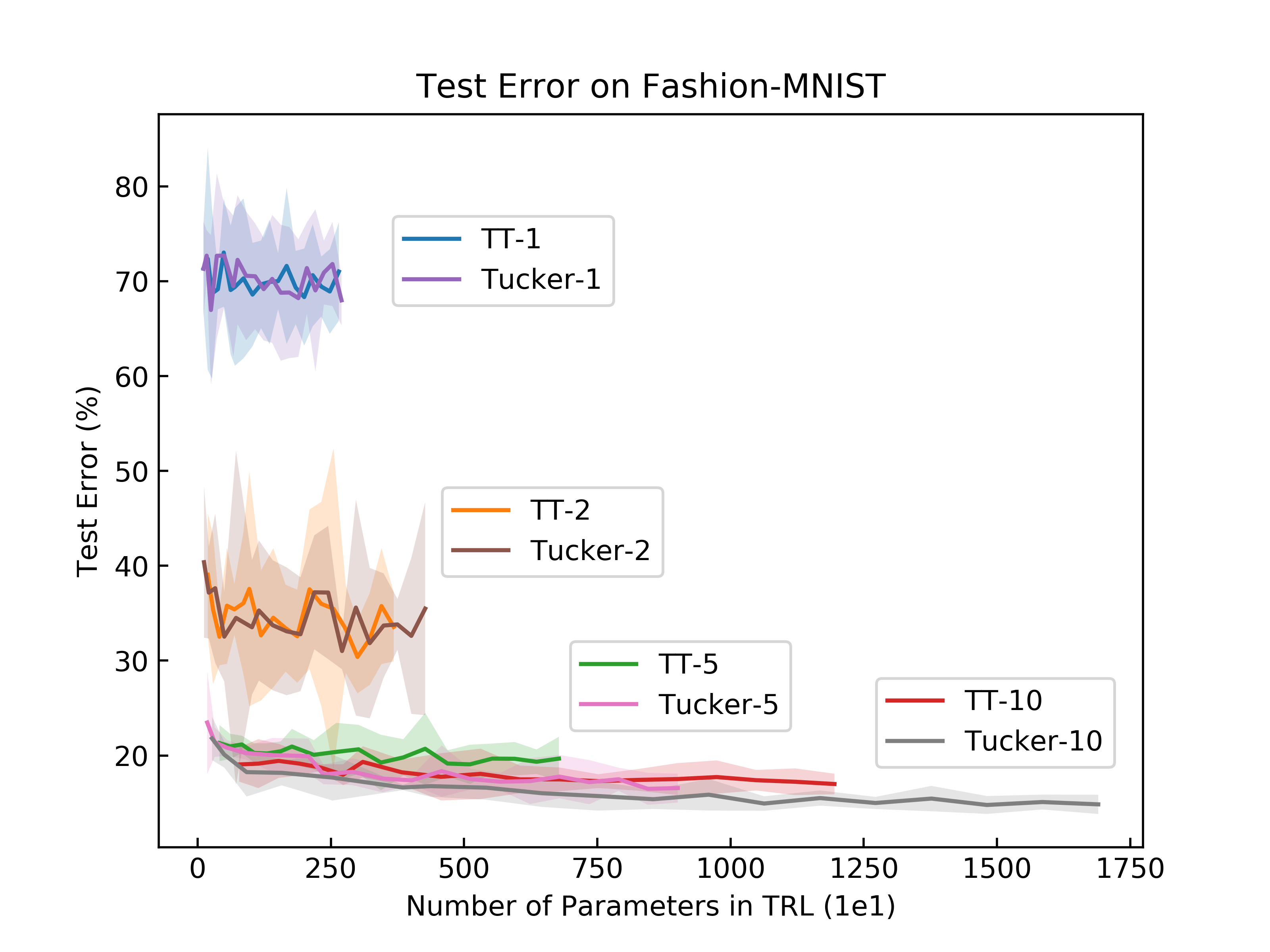}
  
  \caption{Comparison of test error of TRL via TT and Tucker decomposition. We relax the low-rank constraints while fixing the \emph{bottleneck rank} (see Definition~\ref{df:bottleneck_rank}) of both TT and Tucker decomposition. TT-n (resp.~Tucker-n) denotes a model with the \emph{bottleneck rank} to be n (resp.~n). \emph{Left:} MNIST dataset. \emph{Right:} Fashion-MNIST dataset.}  \label{fig:mnist_comaprison}
\end{figure}

\subsection{CIFAR-10 and CIFAR-100 dataset with Residual Networks} \label{sc:exp_cifar10}

We evaluate the performance of tensor regression layer with another benchmark dataset; CIFAR-10 and CIFAR-100 with deep CNNs. CIFAR-10 dataset \cite{krizhevsky2009learning} consists of $60$k training and $10$k test images from 10 classes; airplane, automobile, bird, cat, deer, dog, frog, horse, ship and truck. Similarly, CIFAR-100 consists of colored images of 100 classes \cite{krizhevsky2009learning}. We employ Residual structure network \cite{he2016identity} and replaced the GAP layer with CP, Tucker and TT-TRL. Following~\cite{he2016identity}, we trained the model with initial learning rate of $0.1$ with momentum of $0.9$. The learning rate is multiplied by $0.1$ at $40$k and $60$k iteration steps and the training process is terminated at $80$k steps. The size of each batch was set to  $128$. We set the weight decay to  $0.0002$. The image is pre-processed with whitening normalization followed by the random horizontal flip and cropping with padding size of 2 pixels on each side.

The experimental result are reported in Table~\ref{tb:resnet} for a $32$-layer Residual network \cite{he2016identity} on CIFAR-10 and a $164$-layer ResNet on CIFAR-100. In order to compare the compression rate, we set the baseline model to be the Residual network with fully-connected layer instead of GAP layer. The errors in  Table~\ref{tb:resnet} are obtained by choosing the with the best validation score. The experiment shows that CP-TRL achieves comparable test accuracy to ResNet with GAP layer, however, GAP layer performed the best in terms of both compression and accuracy.

\begin{table}
\begin{tabular}{ l  c  ccccc c} 
 \toprule
 
\multirow{2}{*}{Layer Type} & \multirow{2}{*}{ Rank } & \multicolumn{3}{c}{CIFAR-10} & \multicolumn{3}{c}{CIFAR-100}\\
\cmidrule{3-8}
 & & Vali & Test &CR& Vali & Test &CR \\
 
 \midrule
  FC 	&- & 8.36 & 8.28 & 1.0 & 36.68 & 36.36 &1.0 \\
  GAP 	&- & 7.62 & 8.18 & 64.0 & 29.68 & \bf{29.42} &64.0 \\
  
  \midrule
  
  \multirow{3}{*}{CP-TRL} &5  & 8.32 & 8.43 &91.0 & 34.64 & 36.01 &455.1\\
                          &50 & 8.18 & 8.11 &9.1 & 30.92 & 30.73 &45.5\\
                          &100  & 8.42 & \bf{8.05} &4.6& 31.28  & 31.72 &22.7\\ 
 
 \midrule
 
  \multirow{3}{*}{Tucker-TRL} &$(4,4,4,L)$  & 8.30 & 8.39 & 41.0 & 33.34 & 32.26 &24.5\\
                              &$(8,8,8,L)$  & 7.78 & 8.39 & 7.0 & 30.86 & 31.53 &6.6\\
                              &$(8,8,64,L)$ & 7.92 & 8.58 & 0.9 & TF & TF &$\approx$1.0\\
 
 \midrule
 
  \multirow{3}{*}{TT-TRL} &$(1,1,1,L,1)$  & 8.18 & 8.47 & 54.2 & 31.12 & 30.95 &25.0\\
                          &$(1,8,8,L,1)$  & 7.86 & 9.13 & 7.1  & 30.28 & 31.08 &6.6\\
                          &$(1,8,64,L,1)$ & 8.36 & 8.56 & 0.9  & 31.64 & 32.64 &$\approx$1.0\\

 \bottomrule 
\end{tabular}
\caption{Comparisons of errors (\%) of ResNet-32 (resp. ResNet-164) model with different layers after the final convolutional layer on CIFAR-10 (resp. CIFAR-100) dataset. The training error at the termination of training stage for all models resulted to be $0.0$. L denotes 10 and 100 for CIFAR-10 and CIFAR-100 dataset respectively. FC = Fully-Connected, TF = Training Failed, and CR = Compression Rate.} \label{tb:resnet}
\end{table}

\begin{table}[t]
\centering
\resizebox{\textwidth}{!}{\subfloat[MNIST]{
\begin{tabular}{ l lcc lcc lcc lcc }
\toprule
\hfil\multirow{3}{*}{Layer}\hfill		&\multicolumn{12}{c}{Size of Training Data}	\\
                     							\cmidrule{2-13}
									& \multicolumn{3}{c}{100} 	& \multicolumn{3}{c}{500} 	& \multicolumn{3}{c}{2,000} & \multicolumn{3}{c}{15,000} 			\\
                                    \cmidrule{2-13}
                                    & Rank & Vali & Test		& Rank & Vali & Test		& Rank & Vali & Test		& Rank & Vali & Test					\\
\midrule
FC 				&- 		& 24.48 & 20.73	&- 		& 11.48	&7.95	&-		& 4.18 & 3.86 &-		& 1.64 & 1.33				\\
FC-L2 			&- 		& 20.00 & 19.25	&- 		& 8.00	&7.86	&-		& 3.24 & 3.31 &-		& 1.50 & 1.35				\\
FC DO 			&- 		& 17.26 & 17.86	&- 		& 5.80	&5.98	&-		& 2.86 & 2.53 &-		& 1.16 & 1.21				\\
GAP 			&- 		& 25.30 & 22.10	&- 		& 11.02	&10.46	&-		& 4.54 & 4.85 &-		& 2.78 & 3.15				\\
\midrule				
CP-TRL 			&10				& 25.56 & 17.88		& 30			& 11.80	& 8.91		& 30			& 4.42 & 4.06		& 30			& 2.30	& 2.04		\\
CP-TRL DO 		&10				& 14.82 & 16.06		& 30			& 6.28	& 5.91		& 30			& 2.42 & 2.75		& 30			& 1.82	& 1.38		\\
Tucker-TRL		&[7,7,30,10]	& 25.98 & 22.45		& [7,7,32,10]	& 9.04	& 8.29		& [7,7,7,10]	& 4.08 & 3.73		& [7,7,15,10]	& 1.76	& 1.65		\\
Tucker-TRL DO 	&[7,7,7,10]		& 12.70 & \bf{13.11}& [7,7,30,10]	& 5.16	& 5.49		& [7,7,32,10]	& 2.56 & \bf{2.28}	& [7,7,30,10]	& 1.26	& \bf{1.17}	\\
TT-TRL 			&[1,7,15,10,1]	& 23.24 & 20.39		& [1,7,30,10,1]	& 9.08	& 8.59		& [1,7,15,10,1]	& 3.64 & 3.82		& [1,7,30,10,1]	& 1.66	& 1.36		\\
TT-TRL DO 		&[1,7,7,10,1]	& 14.96 & 14.85		& [1,7,30,10,1]	& 5.42	& \bf{5.18}	& [1,7,30,10,1]	& 2.42 & \bf{2.28}	& [1,7,32,10,1]	& 1.24	& 1.31		\\
\bottomrule

\end{tabular}\label{tb:regularization_mnist}}
}
\\
\resizebox{\textwidth}{!}{
\subfloat[SVHN]{
\begin{tabular}{ l lcc lcc lcc lcc }
\toprule
\hfil\multirow{3}{*}{Layer}\hfill		&\multicolumn{12}{c}{Size of Training Data}	\\
                     							\cmidrule{2-13}
									& \multicolumn{3}{c}{100} 	& \multicolumn{3}{c}{500} 	& \multicolumn{3}{c}{2,000} & \multicolumn{3}{c}{15,000} 			\\
                                    \cmidrule{2-13}
                                    & Rank & Vali & Test		& Rank & Vali & Test		& Rank & Vali & Test		& Rank & Vali & Test					\\
\midrule
FC 				&- 	&78.30  &78.33  &- &49.74  &49.56  &- &27.92 &28.57  &-& 16.66 & 17.64    \\
FC-L2 			&- 	&76.64  &75.73  &- &46.06  &44.25  &- &26.54 &27.11  &-& 15.68 & 16.80    \\
FC DO 			&- 	&75.12  &74.37  &- &41.20  &41.41  &- &26.14 &28.27  &-& 19.64 & 20.22    \\
GAP 			&- 	&86.36  &84.48  &- &75.38  &73.44  &- &72.48 &69.58  &-& 80.56 & 75.72    \\
\midrule				
CP-TRL 			&30 			&79.02 &77.65 		&30 			&55.58 &57.34 		&10 			&32.46 &33.85 		&30 			& 20.32	& 21.11\\
CP-TRL DO 		&30 			&75.78 &73.52 		&30 			&37.34 &39.85 		&30 			&27.92 &28.90 		&30 			& 21.38	& 21.80\\
Tucker-TRL		&[8,8,16,10] 	&79.78 &80.27 		&[8,8,16,10] 	&43.84 &44.29 		&[8,8,32,10] 	&24.48 &24.70 		&[8,8,16,10] 	& 15.26	& 17.27\\
Tucker-TRL DO 	&[8,8,32,10] 	&72.42 &\bf{71.28} 	&[8,8,32,10] 	&34.84 &\bf{34.71} 	&[8,8,16,10] 	&22.32 &\bf{24.50} 	&[8,8,64,10] 	& 15.08	& \bf{16.23}\\
TT-TRL 			&[1,8,64,10,1] 	&78.64 &77.73 		&[1,8,64,10,1] 	&44.10 &43.78 		&[1,8,64,10,1] 	&24.78 &25.50 		&[1,8,32,10,1] 	& 14.76	& 16.28\\
TT-TRL DO 		&[1,8,64,10,1] 	&73.36 &71.63	 	&[1,8,64,10,1] 	&36.64 &36.37 		&[1,8,64,10,1] 	&22.98 &24.60 		&[1,8,64,10,1] 	& 15.24	& 16.52\\
\bottomrule

\end{tabular}\label{tb:regularization_svhn}}
}
\\
\resizebox{\textwidth}{!}{
\subfloat[CIFAR-10]{
\begin{tabular}{ l lcc lcc lcc lcc }
\toprule
\hfil\multirow{3}{*}{Layer}\hfill		&\multicolumn{12}{c}{Size of Training Data}	\\
                     							\cmidrule{2-13}
									& \multicolumn{3}{c}{100} 	& \multicolumn{3}{c}{500} 	& \multicolumn{3}{c}{2,000} & \multicolumn{3}{c}{15,000} 			\\
                                    \cmidrule{2-13}
                                    & Rank & Vali & Test		& Rank & Vali & Test		& Rank & Vali & Test		& Rank & Vali & Test					\\
\midrule
FC 				&- 	& 78.14 & 76.39 		&- & 64.44 & 61.82 &- &27.92  &28.57  &-& 40.70 & 41.88  \\
FC-L2 			&- 	& 76.86 & 75.77 		&- & 62.48 & 61.66 &- &26.54  &27.11  &-& 39.54 & 41.40  \\
FC DO 			&- 	& 73.58 & 73.93 		&- & 61.38 & 61.26 &- &26.14  &28.27  &-& 42.08 & 42.33  \\
GAP 			&- 	& 71.90 & \bf{71.74} 	&- & 60.92 & 60.14 &- &72.48  &69.58  &-& 57.26 & 57.56  \\
\midrule				
CP-TRL 			&10    			& 77.94 & 77.76	& 30  			& 67.84  & 67.05    &10    			&32.46 &33.85    	& 30   			& 45.16  & 46.59 \\
CP-TRL DO 		&10    			& 74.46 & 75.29	& 30  			& 62.40  & 61.11    &30    			&27.92 &28.90    	& 30   			& 47.20  & 48.57\\
Tucker-TRL		&[8,8,64,10]    & 77.60 & 77.21	& [8,8,8,10]  	& 63.84  & 64.63    &[8,8,32,10]    &24.48 &24.70   	& [8,8,64,10]   & 40.28  & 41.36\\
Tucker-TRL DO 	&[8,8,32,10]    & 74.02 & 73.59	& [8,8,32,10]  	& 59.58  & 58.54    &[8,8,16,10]    &22.32 &\bf{24.50} 	& [8,8,32,10]   & 40.20  & 40.51\\
TT-TRL 			&[1,8,8,10,1]   & 74.78 & 75.11	& [1,8,64,10,1] & 63.94  & 62.70    &[1,8,64,10,1]  &24.78 &25.50    	& [1,8,64,10,1] & 38.16  & \bf{38.57}\\
TT-TRL DO 		&[1,8,32,10,1]	& 72.44 & 73.51	& [1,8,32,10,1] & 57.88  &\bf{58.30}&[1,8,64,10,1]  &22.98 &24.60    	& [1,8,64,10,1] & 38.38  & 38.64\\
\bottomrule

\end{tabular}\label{tb:regularization_cifar10}}
}

\caption{On the regularization effect of TRL. The comparison of test/validation errors (\%) with different numbers of training samples is provided. Each entry in the column \# Train refers to the number of samples used to train each model. We used the same $5$k validation samples to select the best model for all experiments. FC-L2 = FC layer with L2 regularization. DO = Dropout.}\label{tb:regularization}
\end{table}

\subsection{On the regularization effect of TRL} \label{sc:exp_reg}

In this section, we investigate the performance of TRL focusing on its function as a regularization to convolutional neural networks. We used shallow CNNs with different train/validation split where the number of the training samples were kept to be small. We compare the performance of TRL with fully-connected layer and GAP layer. To improve the regularization performance, Dropout \cite{srivastava2014dropout}  and weight decay were included in the comparison. The training datasets are obtained by randomly selecting samples from the initial training dataset, and keeping $5$k samples for validation  for each train/validation split.

We evaluate the performance of each model on three datasets; MNIST, Street View House Numbers (SVHN) \cite{netzer2011reading} and CIFAR-10. SVHN dataset consists of colored images of house numbers where it contains $73$k and $26$k samples for training and testing respectively. We employed a CNN with two (resp. three) convolutional layers for MNIST dataset (resp. CIFAR-10 and SVHN dataset). The dropout is inserted after the final convolutional layer.

The rank of each TRL is selected based on the dimensions of the output tensor as in Section~\ref{sc:exp_mnist}. We run experiments with early stopping for all experiments where the maximum steps is set to $100k$ for MNIST and to $10k$ for SVHN and CIFAR-10. The best rate for dropout is selected based on the validation accuracy where the hyper-parameter is samples from $\{0.3, 0.5, 0.7\}$. The decay factor for L2-regularization is similarly chosen from the set $\{0.01,0.001,0.0001\}$.

The outcome of the experiment is presented in Table~\ref{tb:regularization}. An unique behavior of TRL is observed in Table~\ref{tb:regularization} where in most of the settings using dropout with Tucker and TT-TRL achieves better test accuracy than using dropout with a fully-connected layer.

\section{Conclusion} \label{sc:conc}
Tensor regression layer replaces the last flattening operation and fully connected layers with tensor regression with low tensor rank structure. We investigate tensor regression layer with various tensor decompositions. TRL with CP, Tucker and TT decompositions were presented and investigated in this work. We show that the learning procedure for each type of tensor regression layer can be derived using tensor algebra. An analysis on the upper bound of the dimension of the image of the regression function is presented, where we show that the rank of Tucker decomposition and \emph{TT ranks} affect such dimension.

We evaluated proposed models using benchmark dataset (i.e. handwritten digits and natural images). We did not observe significant differences in accuracy among TRLs with various decompositions for MNIST and CIFAR-10 dataset. The result using the state-of-the-art deep convolutional model shows that when compared to a baseline model with fully-connected layer, TRL with CP decomposition achieved the rate of compression $91.0$ with the sacrifice of accuracy $0.3\%$. When compared to the Residual network with GAP layer, our model empirically exhibits comparable performance in both accuracy and compression rate.

\clearpage
\newpage

\bibliographystyle{plain}
\bibliography{trl}

\end{document}